\def\year{2021}\relax
\documentclass[letterpaper]{article} 
\usepackage{aaai21}  
\usepackage{times}  
\usepackage{helvet} 
\usepackage{courier}  
\usepackage[hyphens]{url}  
\usepackage{graphicx} 
\usepackage{natbib}  
\usepackage{caption} 
\usepackage{makecell}
\usepackage{multirow}
\usepackage{booktabs}
\usepackage{upgreek}
\usepackage[ruled,vlined,noend,linesnumbered]{algorithm2e}
\usepackage{amsmath}
\usepackage{amssymb}
\usepackage{amsthm}
\usepackage{mathtools}
\usepackage{subcaption}

\theoremstyle{definition}
\newtheorem{definition}{Definition}
\newtheorem{theorem}{Theorem}
\newtheorem{lemma}{Lemma}
\newtheorem{assumption}{Assumption}

\newcommand{\normal}[1]{\normalfont{\text{#1}}}
\renewcommand{\v}{\normalfont{\textbf{v}}}
\newcommand{\w}{\normalfont{\textbf{w}}}
\newcommand{\x}{\normalfont{\textbf{x}}}
\newcommand{\z}{\normalfont{\textbf{z}}}
\newcommand{\g}{\nabla}

\DeclarePairedDelimiter\p{\lparen}{\rparen}
\DeclarePairedDelimiter\abs{\lvert}{\rvert}   
\DeclarePairedDelimiter\norm{\lVert}{\rVert}  
\DeclarePairedDelimiter\bkt{[}{]}             
\DeclarePairedDelimiter\set{\{}{\}}           

\DeclarePairedDelimiter\floor{\lfloor}{\rfloor}

\def\gE{\mathbb{E}}

\def\gP{\mathbb{P}}

\def\gD{\mathcal{D}}
\def\gN{\mathcal{N}}
\def\gG{\mathcal{G}}

\DeclareMathOperator*{\argmin}{arg\,min}

\title{TornadoAggregate: Accurate and Scalable Federated Learning\\via the Ring-Based Architecture}
\author {
    Jin-woo Lee,\textsuperscript{\rm 1}
    Jaehoon Oh, \textsuperscript{\rm 1}
    Sungsu Lim, \textsuperscript{\rm 2}
    Se-Young Yun, \textsuperscript{\rm 1}
    Jae-Gil Lee, \textsuperscript{\rm 1} \\
}
\affiliations {
    \textsuperscript{\rm 1} Korea Advanced Institute of Science and Technology \\
    \textsuperscript{\rm 2} Chungnam National University \\
    jinwoo.lee@kaist.ac.kr, jhoon.oh@kaist.ac.kr, sungsu@cnu.ac.kr, yunseyoung@kaist.ac.kr, jaegil@kaist.ac.kr
}
\begin{document}

\maketitle

\begin{abstract}
Federated learning has emerged as a new paradigm of collaborative machine learning; however, many prior studies have used global aggregation along a star topology without much consideration of the communication scalability or the diurnal property relied on clients' local time variety. In contrast, ring architecture can resolve the scalability issue and even satisfy the diurnal property by iterating nodes without an aggregation. Nevertheless, such ring-based algorithms can inherently suffer from the high-variance problem. To this end, we propose a novel algorithm called \textbf{\textit{TornadoAggregate}} that improves both accuracy and scalability by facilitating the ring architecture. In particular, to improve the accuracy, we reformulate the loss minimization into a variance reduction problem and establish three principles to reduce variance: Ring-Aware Grouping, Small Ring, and Ring Chaining. Experimental results show that TornadoAggregate improved the test accuracy by up to $26.7\%$ and achieved near-linear scalability.
\end{abstract}

\section{Introduction}

\emph{Federated learning}\,\cite{FSVRG,FedAvg} enables mobile devices to collaboratively learn a shared model while keeping all training data on the devices, thus avoiding data transfer to the cloud or central server. One of the main reasons for this recent boom in federated learning is that it does not compromise user privacy. In this framework, star architecture\,(Figure~\ref{fig:Architecture}(a)), which involves a central parameter server aggregating and broadcasting locally learned models, has been most widely adopted in favor of its simple distributed parallelism. However, the star architecture can easily become a communication bottleneck and cannot take into account the diurnal property of federated learning\,\cite{FedAvg2,SemiCyclic}, in which the global data distribution of clients significantly varies due to the difference in the clients' local time.

Ring architecture\,(Figure~\ref{fig:Architecture}(b)), in contrast, can resolve the scalability issue and even satisfy the diurnal property by iterating nodes without a central coordinator. In addition, it has the potential to improve accuracy through an unbiased estimation of conventional centralized learning at the expense of communication cost. Notably, \citet{Astraea} proposed a star architecture with ring-based groups, while \citet{MM-PSGD} proposed a ring architecture with stars-based groups. \citet{IFCA} and \citet{SemiCyclic} proposed star-based and ring-based groups, respectively, without global communication. Other than the importance of addressing the problem, less has been addressed how a ring-based architecture should be developed from the perspective of both accuracy and scalability.




\begin{figure}[t!]
    \centering
    \includegraphics[width=.8\columnwidth]{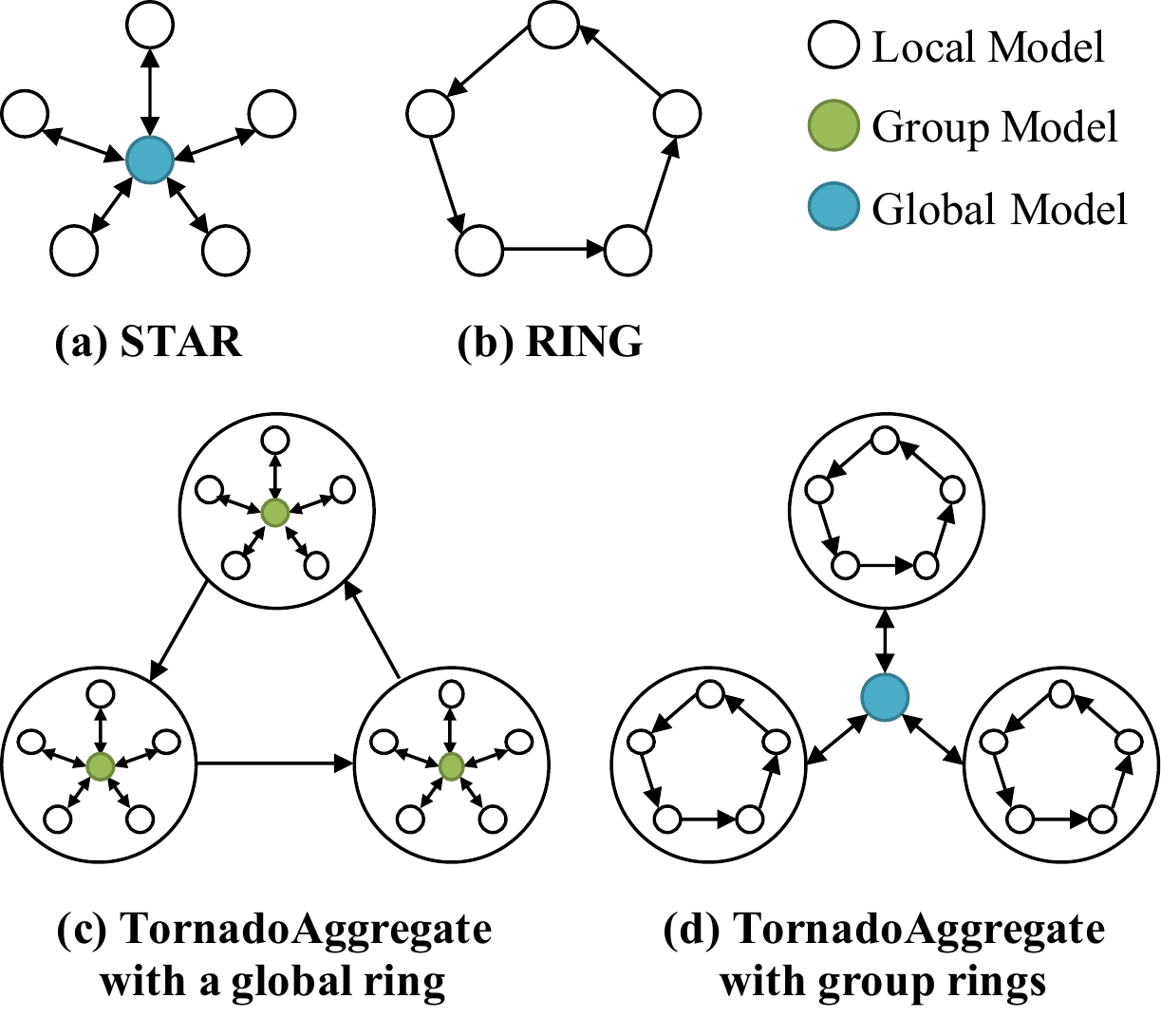}
    \caption{\textbf{Representative architectures and proposed ring-based algorithm TornadoAggregate}.}
    \label{fig:Architecture}
\end{figure}

In this paper, we propose a novel \textbf{\textit{TornadoAggregate}} algorithm that improves both accuracy and scalability by facilitating the ring architecture. To improve accuracy, in particular, TornadoAggregate aims at reducing the variance inherent in a ring iteration by considering three principles: \emph{ring-aware grouping}, \emph{small ring}, and \emph{ring chaining}. Based on the \emph{ring-aware grouping} principle, for the TornadoAggregate with a global ring\,(Figure~\ref{fig:Architecture}(c)) and group rings\,(Figure~\ref{fig:Architecture}(d)), nodes are grouped such that it reduces the inter-group variance in the global ring and inter-node variance in each group ring, respectively; the number of groups is adjusted to satisfy the \emph{small ring} principle, thus achieving the small variance; we introduce \emph{ring chaining} technique to increase the batch size with high node utilization in a ring, leading to the reduced variance.

We confirmed that TornadoAggregate achieved a higher accuracy by up to $26.7\%$ and near-linear scalability.


\section{Architectures of Federated Learning}
We first briefly describe \emph{federated learning} and then survey relevant architectures in terms of \emph{accuracy} and \emph{scalability}.

\begin{table}[t!]
\centering
\small
\begin{tabular}{ccc} \toprule
    \hspace{-3pt}\textbf{Architecture}\hspace{-10pt} & \textbf{Convergence Bound}\hspace{-10pt} & \makecell{\textbf{Communication}\\\textbf{Scalability}}\hspace{-5pt} \\ \midrule
    \textit{STAR} & O($D h(\uptau)$) & O($\abs{\gN}$) \\
    \textit{RING} & 0\,(Approximate) & O($1$) \\
    \hspace{-3pt}\textit{STAR-stars} & O($\Delta h(\uptau_1\uptau_2) + \delta h(\uptau_1)$) & O($\abs{\gN}$) \\
    \hspace{-3pt}\textit{STAR-rings} & O($D h(\uptau_1\uptau_2)$)\,(Approximate) & O($\abs{\gG}$) \\
    \hspace{-3pt}\textit{RING-stars} & O($D h(\uptau_1)$)\,(Approximate) & O($\abs{\gN}/\abs{\gG}$) \\
    \hspace{-3pt}\textit{RING-rings} & 0\,(Approximate) & O($1$) \\
    \textit{stars} & O($\delta h(\uptau)$) & O($\abs{\gN}$)) \\
    \textit{rings} & 0\,(Approximate) & O($\abs{\gG}$) \\ \bottomrule
\end{tabular}
\caption{\textbf{Comparison of architectures}. $\abs{\gN}$ and $\abs{\gG}$ denote the number of nodes and groups, respectively. $D$, $\delta$, and $\Delta$ denote the local-to-global, local-to-group, and group-to-global divergence, respectively.}
\label{tbl:Architecture}
\end{table}

\subsection{Basics of Federated Learning}
The objective of \emph{federated learning} is to find an approximate solution of Eq.~\eqref{eq:Problem}\,\cite{FedAvg}. Here, $F(\w)$ is the loss of predictions with a model $\w$ over the set of all data examples $\gD \triangleq \cup_{i \in \gN}{\gD^{i}}$ across all nodes, where $\gN$ is the set of node indices, and $F^{i}\p{\w} \triangleq \sum_{\p{\x, y} \in \gD^{i}}{\frac{1}{\abs{\gD^{i}}} l\p{\w, \x, y}}$ is the loss of predictions with a loss function $l$ parameterized by $\w$ over the set of data examples $\p{\x, y} \in \gD^{i}$ on node $i$.
\begin{equation}
\label{eq:Problem}
\resizebox{.9\columnwidth}{!}{$
    \min_{ \w \in \mathbb{R}^d }{ F\p{\w} } ~~ \normal{where} ~~ F\p{\w} \triangleq \sum_{ i \in \gN }{ \frac{ \abs{\gD^{i}} }{ \abs{\gD} } F^{i}\p{\w} }
$}
\end{equation}

To resolve Eq.~\eqref{eq:Problem}, a huge number of architectures are being actively proposed, and based on hierarchical composition, they can be classified into \emph{three} main categories: \emph{flat}, \emph{consensus group}, and \emph{pluralistic group}. Table~\ref{tbl:Architecture} compares them in terms of convergence bound and scalability, each of which is analyzed in Appendix~\ref{sec:ConvergenceAnalysis} and \ref{sec:CommScalability}, respectively.

\subsection{Flat Architecture}
\emph{Flat} represents an architecture without hierarchical composition that, in turn, includes \textit{STAR} and \textit{RING} architecture. \textit{STAR} is the same as the canonical FedAvg\,\cite{FedAvg} without node sampling, as defined by Definition~\ref{def:STAR}.
\begin{definition}
\label{def:STAR}
\textbf{\textit{STAR}} involves local update, which learns each local model $\w^{i}$ with learning rate $\eta$ by performing gradient descent steps, and \emph{global aggregation}, which learns the global model $\w$ by aggregating all $\w^{i}$ along a star topology and synchronizes $\w^{i}$ with $\w$ every $\uptau$ epochs, as in Eq.~\eqref{eq:w_STAR}.
\begin{equation}
\label{eq:w_STAR}
\resizebox{.9\columnwidth}{!}{$
\begin{multlined}
    \w^{i}_{t} \triangleq \begin{dcases}
        \w^{i}_{t-1} - \eta\g F^{i}\p{\w^{i}_{t-1}} & \normal{if} ~ t \normal{ mod } \uptau \ne 0 \\
        \w_{t} & \normal{if} ~ t \normal{ mod } \uptau = 0
    \end{dcases} \\
    \normal{where} ~~~ \w_{t} \triangleq \sum\limits_{ i \in \gN }{ \frac{ \abs{\gD^{i}} }{ \abs{\gD} }\bkt{ \w^{i}_{t-1} - \eta\g F^{i}\p{\w^{i}_{t-1}} } }
\end{multlined}
$}
\end{equation}
\end{definition}
\noindent \textit{STAR} exhibits the most simple distributed parallelism that, at the same time, leads to low scalability of O($\abs{\gN}$) due to the communication bottleneck in global aggregation.

In contrast, \textit{RING}\,\cite{PipeSGD,SemiCyclic} resolves the aforementioned scalability issue by removing the global aggregation, as defined by Definition~\ref{def:RING}.
\begin{definition}
\label{def:RING}
\textbf{\textit{RING}} extends \textit{STAR} by replacing the global aggregation with \emph{global inter-node transfer} that synchronizes a new model $\w^{i_{j}}$ on node $i_{j}$ with the previously learned model $\w^{i_{j-1}}$ on node $i_{j-1}$ every $\uptau$ epoch along a certain ring topology\footnote{As long as the conditions are satisfied, a ring can be defined in any way\,(e.g., a random permutation in Algorithm~\ref{alg:Tornado}).} $[i_{j} \in \gN | j = \floor{t / \uptau}, i_{j+\abs{\gN}} = i_{j}]$ with a period $\abs{\gN}$ to satisfy the diurnal property, as in Eq.~\eqref{eq:w_RING}.
\begin{equation}
\label{eq:w_RING}
\begin{multlined}
    \w^{i_{j}}_{t} \triangleq \begin{dcases}
        \w^{i_{j}}_{t-1} - \eta\g F^{i_{j}}\p{\w^{i_{j}}_{t-1}} & \normal{if} ~ t \normal{ mod } \uptau \ne 0 \\
        \w^{i_{j-1}}_{t} & \normal{if} ~ t \normal{ mod } \uptau = 0
    \end{dcases}
\end{multlined}
\end{equation}
\end{definition}
\noindent \textit{RING} exhibits low convergence bound, attributed to Theorem~\ref{thm:RING}, and benefits from high scalability of O($1$). However, it is considered impractical in federated learning, where a large $\abs{\gN}$ is assumed, because it takes $\abs{\gN}$ times as much communication rounds to iterate a global epoch as \textit{STAR}.
\begin{theorem}
\label{thm:RING}
\textit{RING} is an unbiased estimator of the centralized learning that learns a centralized model by assuming the federated datasets to be located at centralized storage.
\end{theorem}
\noindent Owing to lack of space, we defer all proofs to Appendix~\ref{sec:Proof}.

\subsection{Consensus Group Architecture}
\emph{Consensus group} represents an architecture with group hierarchy and global communication to reach a global consensus among groups, which, in turn, includes \emph{four} architectural combinations: \textit{STAR-stars}, \textit{STAR-rings}, \textit{RING-stars}, and \textit{RING-rings}. First, \textit{STAR-stars}\,\cite{lin2018don,bonawitz2019towards,HierFAVG,HFEL,HeteroCell} mitigates the non-IID issue via group-based learning\,\cite{IIDSharing}, which leads to improved accuracy, as defined by Definition~\ref{def:STAR-stars}.
\begin{definition}
\label{def:STAR-stars}
\textbf{\textit{STAR-stars}} extends \textit{STAR} by additionally allowing multiple intermediate group star aggregations, thus postfixed by \textit{stars}. In particular, the set of all node indices $\gN$ is partitioned into sets of node indices for $\abs{\gG}$ node groups $\set{\gN^{k}}_{k=1 \ldots \abs{\gG}}$, where $\cup_{ k \in \gG }{ \gN^{k} } = \gN$ and $\forall k \ne l$, $\gN^{k} \cap \gN^l = \emptyset$, and $\gD^{k} \triangleq \cup_{i \in \gN^{k}}{\gD^{k,i}}$. Then, for each group, it learns the group model $\w^{k}$ by aggregating all local models $\w^{k,i}$ along a group star topology and synchronizes $\w^{k,i}$ with $\w^{k}$ every $\uptau_1$ epochs, as shown by Eq.~\eqref{eq:w_STAR-stars}. Similar to Eq.~\eqref{eq:w_STAR}, a global aggregation is performed every $\uptau_1\uptau_2$ steps.
\begin{equation}
\label{eq:w_STAR-stars}
\begin{aligned}
    \w^{k,i}_{t} \triangleq & \begin{dcases}
        \w^{k,i}_{t-1} - \eta\g F^{k,i}\p{\w^{k,i}_{t-1}} & \normal{if} ~ t \normal{ mod } \uptau_1 \ne 0 \\
        \w^{k}_{t} & \normal{if} ~ \begin{aligned}[t]
            & t \normal{ mod } \uptau_1 = 0, \\
            & t \normal{ mod } \uptau_1\uptau_2 \ne 0
        \end{aligned} \\
        \w_{t} & \normal{if} ~ t \normal{ mod } \uptau_1\uptau_2 = 0
    \end{dcases} \\
    \mbox{where} & ~~~ \w^{k}_{t} \triangleq \sum\limits_{ i \in \gN^{k} }{ \frac{ \abs{\gD^{k,i}} }{ \abs{\gD^{k}} }\bkt{ \w^{k,i}_{t-1} - \eta\g F^{k,i}\p{\w^{k,i}_{t-1}} } } \\
    \mbox{and} & ~~~ \w_{t} \triangleq \sum\limits_{ k \in \gG }{ \frac{ \abs{\gD^{k}} }{ \abs{\gD} }\w^{k}_{t} }
\end{aligned}
\end{equation}
\end{definition}
\noindent \textit{STAR-stars} is known to improve \textit{STAR} under certain parameter settings in favor of the non-IID mitigation\,\cite{HierFAVG}, but it exhibits the low scalability of O($\abs{\gN}$).


Next, analogous to the development of \textit{RING}, \textit{STAR-rings}\,\cite{Astraea}, \textit{RING-stars}\,\cite{TurboAggregate,MM-PSGD}, and \textit{RING-rings}\,\cite{SemiCyclic} also aim at improving both convergence bound and scalability while sacrificing communication cost, which are summarized in Table~\ref{tbl:Architecture}. Formal definitions are as follows.
\begin{definition}
\label{def:STAR-rings}
Similar to Definition~\ref{def:RING}, \textbf{\textit{STAR-rings}} extends \textit{STAR-stars} by replacing the group aggregation with \emph{group inter-node transfer} that, for each group $k \in \gG$, synchronizes a new model $\w^{k,i_{j}}$ on node $i_{j} \in \gN^{k}$ with the previously learned model $\w^{k,i_{j-1}}$ on node $i_{j-1}$ every $\uptau_1$ epochs along a certain ring topology $[i_{j} \in \gN^{k} | j = \floor{t / \uptau_1}, i_{j+\abs{\gN^{k}}} = i_{j}]$ with a period $\abs{\gN^{k}}$ to satisfy the diurnal property within each group. In short, $\w^{k,i}_{t} \triangleq \w^{k}_{t}$\,(group aggregation) of Eq.~\eqref{eq:w_STAR-stars} is replaced with $\w^{k,i_{j}}_{t} \triangleq \w^{k,i_{j-1}}_{t}$\,(group inter-node transfer).
\end{definition}
%

\begin{definition}
\label{def:RING-stars}
Similar to Definition~\ref{def:STAR-rings}, \textbf{\textit{RING-stars}} extends \textit{STAR-stars} by replacing the global aggregation with \emph{global inter-group transfer} that synchronizes a new local model $\w^{k_{l},i}$ on node $i \in \gN^{k_{l}}$ with the previously learned group model $\w^{k_{l-1}}$ in group $k_{l-1}$ every $\uptau_1\uptau_2$ steps along a certain ring topology $[k_{l} \in \gG | l = \floor{t / \p{\uptau_1\uptau_2}}, k_{l+\abs{\gG}} = k_{l}]$ with a period $\abs{\gG}$ to satisfy the diurnal property across all groups. In short, $\w^{k,i}_{t} \triangleq \w_{t}$\,(global aggregation) of Eq.~\eqref{eq:w_STAR-stars} is replaced with $\w^{k_{l},i}_{t} \triangleq \w^{k_{l-1}}_{t}$\,(global inter-group transfer).
\end{definition}
%
%

\begin{definition}
\textbf{\textit{RING-rings}} extends \textit{STAR-stars} by replacing the group and global aggregation with the group inter-node and global inter-group transfer, respectively.
\end{definition}
%
%
\noindent Similar to \textit{RING}, \textit{RING-rings} is considered impractical in federated learning due to large number of nodes.

\subsection{Pluralistic Group Architecture}
\emph{Pluralistic group} represents an architecture with group hierarchy and without global communication to develop more independent and specialized group models than the aforementioned consensus model, which leads to decreased non-IIDness, and consequently, improved accuracy. Representative \emph{pluralistic group} includes \textit{stars} and \textit{rings}. Recently, \emph{stars}\,\cite{RobustHetero,IFCA,FeSEM,FLHC,RecursiveBi} has received great attention, which is defined by Definition~\ref{def:stars}.
\begin{definition}
\label{def:stars}
\textbf{\textit{stars}} is defined as \textit{STAR-stars} without global aggregation. Unlike Eq.~\eqref{eq:w_STAR-stars}, $\w^{k,i}_{t}$ is not synchronized with $\w_{t}$ and thus group communication rounds $\uptau_2$ is not defined.
\end{definition}
\begin{definition}
\label{def:rings}
Similar to Definition~\ref{def:stars}, \textbf{\textit{rings}} is defined as \textit{STAR-rings} without global aggregation.
\end{definition}
\noindent It is important to note that the growing popularity of pluralistic group architectures may be hype. According to Theorem~\ref{thm:stars}, \textit{stars} may achieve lower accuracy than \textit{STAR}.
\begin{theorem}
\label{thm:stars}
The convergence bound O($\delta h(\uptau)$) of \textit{stars} doesn't necessarily be better than O($D h(\uptau)$) of \textit{STAR}.
\end{theorem}

%
%
\noindent Lastly, as shown in Table~\ref{tbl:Architecture}, \textit{rings} can benefit from the low convergence bound as well as high scalability of O($\abs{\gG}$).
%

\section{Reformulation: Variance Reduction}
As previously noted, ring-based architectures such as \textit{RING}, \textit{STAR-rings}, \textit{RING-stars}, \textit{Ring-rings}, and \textit{rings} have great potential to improve both accuracy and scalability. However, the convergence analysis framework introduced in this study is mostly based on unbiasedness property to easily compare all of the architectures. To better understand architectures from the perspective of accuracy, the variance should also be further considered. To this end, based on Theorem~\ref{thm:RING-Variance}, we reformulate the problem of Eq.~\eqref{eq:Problem} to the variance reduction of ring-based architectures.
\begin{theorem}
\label{thm:RING-Variance}
\textit{RING} exhibits higher variance than the centralized learning\,(unbiased estimator of \textit{RING} from Theorem~\ref{thm:RING}).
\end{theorem}

\section{Proposed Algorithm: TornadoAggregate}
\label{sec:Algorithm}
The ring-based federated learning under high variance issue looks similar to the continual learning under catastrophic forgetting\,\cite{parisi2019continual}, but the former additionally involves partitioned data groups as well as data iteration along a ring. Considering the differences, we establish \emph{three} principles to reduce variance.
\begin{itemize}
\item
    \textbf{Principle 1\,(Ring-Aware Grouping)}: For architectures with group rings, nodes should be clustered so that inter-node variance becomes low within a group. On the other hand, for architectures with a global ring, nodes should be IID grouped so that inter-group variance becomes low.
\item
    \textbf{Principle 2\,(Small Ring)}: It is straightforward that, the smaller a ring, the lower its iteration variance.
\item
    \textbf{Principle 3\,(Ring Chaining)}: A ring can have multiple iteration chains\,\cite{MM-PSGD}, each of which iterates the same ring at a different starting node and thus learns an unbiased model different from each other. Multiple chains can reduce learning variance, which is attributed to the reduced variance from increased batch size.
\end{itemize}

Based on the abovementioned principles, we propose a novel ring-based algorithm called \textbf{\textit{TornadoAggregate}} and derive \emph{two} heuristics according to the architecture type. We refer to TornadoAggregate with \textit{RING-stars} and \textit{STAR-rings} as Tornado and Tornadoes, respectively. In particular, for the \emph{ring-aware grouping} principle, Tornado and Tornadoes require nodes to be IID grouped and clustered, respectively; for the \emph{small ring} principle, Tornado and Tornadoes require a small and large number of groups, respectively; for the \emph{ring chaining} principle, both require a large number of chains.

\begin{algorithm}[t!]
\small
\SetAlgoLined
\SetKwInOut{Input}{\textsc{Input}}
\SetKwInOut{Output}{\textsc{Output}}
\caption{Tornado\,(RING-stars)}
\label{alg:Tornado}
\Input{$\gN$, $\abs{\gG}$, $C$, $\uptau_1$, $\uptau_2$}
\Output{$\w_{T}$}
Initialize $\set{\w^{k,i}_{0}}_{i \in \gN}$ to a random model $\w_{0}$ \\
Initialize a random ring $\bkt{k_{l} \in \gG | l \in \mathbb{N}^0, k_{l+\abs{\gG}} = k_{l}}$ \\
$\set{\gN^{k}}_{k \in \gG} \leftarrow \textsc{Group\_By\_IID}\p{\gN}$ \tcp{Algorithm \ref{alg:Grouping}}
\For{$t \leftarrow 0, \cdots, T-1$}{
    \For{\normal{\textbf{each}} chain $c \leftarrow 0, \cdots, C-1$ \normal{\textbf{in parallel}}}{
        $l \leftarrow \floor{t / \p{\uptau_1\uptau_2}}$, $k \leftarrow \p{k_{l}+c} \normal{ mod } \abs{\gG}$ \\
        \For{\normal{\textbf{each}} node $i \in \gN^{k}$ \normal{\textbf{in parallel}}}{
            $\w^{k,i}_{t+1} \leftarrow \w^{k,i}_{t} - \eta\g F^{k,i}\p{\w^{k,i}_{t}}$
        }
        \If{$t \normal{ mod } \uptau_1 \normal{ \textbf{and} } t \normal{ mod } \uptau_1\uptau_2 \ne 0$}{
            $\set{\w^{k,i}_{t}}_{i \in \gN^{k}} \leftarrow \sum_{i \in \gN^{k}}{\frac{\abs{\gD^{k,i}}}{\abs{\gD^{k}}}\w^{k,i}_{t}}$
        }
        \If{$t \normal{ mod } \uptau_1\uptau_2 = 0$}{
            $k_{next} \leftarrow \p{k_{l+1}+c} \normal{ mod } \abs{\gG}$ \\
            $\set{\w^{k_{next},i}_{t}}_{i \in \gN^{k_{next}}} \leftarrow \sum_{i \in \gN^{k}}{\frac{\abs{\gD^{k,i}}}{\abs{\gD^{k}}}\w^{k,i}_{t}}$
        }
    }
}
\end{algorithm}

As shown in Algorithm \ref{alg:Tornado}, Tornado takes the node set $\gN$ and the number of groups $\abs{\gG}$, chains $C$, epochs $\uptau_1$, and communication rounds $\uptau_2$ as input and returns the final model $\w_{T}$ as output. It begins by initializing all local models $\w^{k,i}_0$, a randomly permuted inter-group ring $\bkt{k_{l}}$, and group indices $\set{\gN^{k}}$ via the IID node grouping\,(Lines 1--3). Then, for each chain and each node, the local updates are performed\,(Lines 5--8); each group model is learned by aggregating all local models every $\uptau_1$ epochs\,(Lines 9--10) and is transferred to all nodes in the next group $k_{next}$ every $\uptau_1\uptau_2$ steps\,(Lines 11--13). Overall, Lines 4--13 repeat for $T$ steps.

Because Tornado and Tornadoes are inherently correlated, we defer the description of Tornadoes to Appendix~\ref{sec:AlgorithmDetail}.

\section{Evaluation}

\subsection{Experimental Setting}

\subsubsection{Benchmark Datasets and Models} We used \emph{two} official benchmark datasets and models provided by FedML.
\begin{itemize}
\item
    \textbf{FedShakespeare on RNN} consists of 715 nodes with 16068 train and 2356 test examples. RNN is the same as the one proposed by \citet{FedAvg}.
\item
    \textbf{MNIST on logistic regression} consists of 1000 nodes with 10 classes of 61664 train and 7371 test examples.
\end{itemize}

\subsubsection{Algorithms}


\begin{table}[t!]
\centering
\small
\begin{tabular}{cccc} \toprule
    \textbf{Algorithm} & \textbf{Architecture} & \makecell{\textbf{Group}\\\textbf{Type}} & \makecell{\textbf{\#}\\\textbf{Chain}} \\ \midrule
    FedAvg(\citeauthor{FedAvg}) \hspace{-5pt} & \textit{STAR} & - & - \\
    IFCA(\citeauthor{IFCA}) & \textit{stars} & Cluster & - \\
    HierFAVG(\citeauthor{HierFAVG}) & \textit{STAR-stars} & Random & - \\
    Astraea(\citeauthor{Astraea}) & \textit{STAR-rings} & IID & 1 \\
    MM-PSGD(\citeauthor{MM-PSGD}) & \textit{RING-stars} & Cluster & 1 \\ \midrule
    \textbf{Tornado\,(Proposed)} & \textit{RING-stars} & IID & $\abs{\gG}$ \\
    \textbf{Tornadoes\,(Proposed)} & \textit{STAR-rings} & Cluster & $\abs{\gG}$ \\ \bottomrule
\end{tabular}
\caption{\textbf{Comparison of algorithms}.}
\label{tbl:Algorithm}
\end{table}

In Table~\ref{tbl:Algorithm}, the proposed Tornado and Tornadoes are compared with \emph{five} state-of-the-art algorithms in terms of group type and number of chains. 

\subsection{Results}

\begin{figure}[t!]
    \centering
    \includegraphics[width=.6\columnwidth]{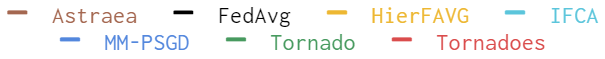}
    \includegraphics[width=.7\columnwidth]{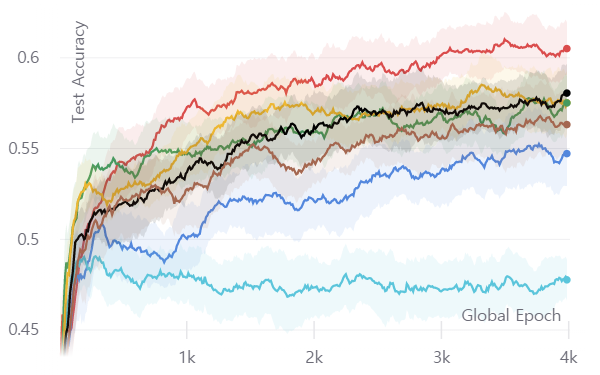}
    \caption{\textbf{Test accuracy for FedShakespeare}.}
\label{fig:FedShakespeare-TestAccuracy}
\end{figure}

Figure~\ref{fig:FedShakespeare-TestAccuracy} shows the test accuracy for FedShakespeare dataset. Tornadoes outperformed the state-of-the-art algorithms by up to $26.7\%$ and the next best group\,(Tornado, FedAvg, and HierFAVG) by $4.4\%$ on average. The low performance of Tornado relative to Tornadoes is attributed to the communication interval; Tornado(\textit{RING-stars}) takes $\uptau_1\uptau_2$ steps for a global inter-group transfer in the \textit{RING}, which is larger than $\uptau_1$ steps of Tornadoes\,(\textit{STAR-rings}) for a group inter-node transfer in each \textit{ring}, thus causing higher divergence. The poor performances of Astraea and MM-PSGD come from the high variance caused by the inappropriate node grouping and low chain utilization. Lastly, IFCA achieved the worst accuracy due to the difficulty of clustering FedShakespeare dataset, as explained by the small clustering cost reduction of only $4.3\%$, in which case the relationship between FedAvg and IFCA is consistent with Theorem~\ref{thm:stars}.

\begin{figure}[t!]
    \centering
    \includegraphics[width=.6\columnwidth]{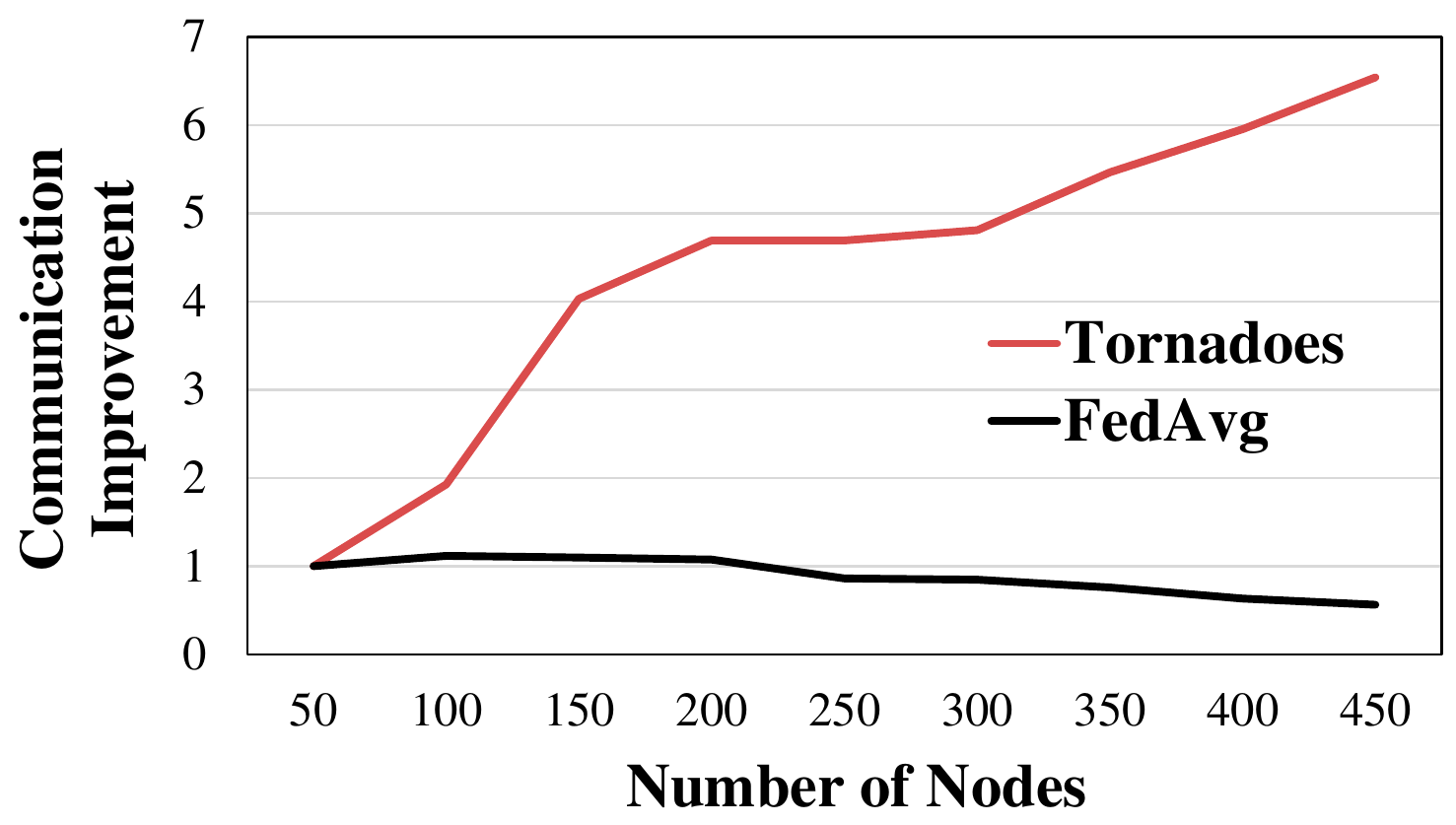}
    \caption{\textbf{Communication scalability to node size}.}
    \label{fig:Scalability}
\end{figure}

Figure~\ref{fig:Scalability} shows the communication scalability to the number of nodes ranging from $50$ to $500$ for MNIST dataset. For each case, we measured the communication data size in bytes to reach the converged train accuracy of the case with $50$ nodes\,($73\%$ for FedAvg and $79\%$ for Tornadoes) and showed the improvement relative to that case. Tornadoes achieved near-linear scalability,
 which is attributed to the superior communication scalability cost of \textit{STAR-rings}.
\vspace{-3pt}

\section{Acknowledgements}
This work was supported by Institute of Information \& Communications Technology Planning \& Evaluation\,(IITP) grant funded by the Korea government\,(MSIT) (No. 2020-0-00862, DB4DL: High-Usability and Performance In-Memory Distributed DBMS for Deep Learning).
\vspace{-3pt}

\section{Conclusion}
In this paper, we provided a comprehensive survey of learning architectures in terms of accuracy and scalability.
Our formal analysis led to the necessity of ring-based architecture and its inherent variance reduction problem. To this end, we proposed a novel ring-based algorithm \textbf{\textit{TornadoAggregate}} that improves both scalability and accuracy by reducing variance in a ring iteration.
Experimental results show that, compared with the state-of-the-art algorithms, TornadoAggregate improved the test accuracy by up to $26.7\%$ and achieved near-linear scalability. Overall, we believe that our novel ring-based algorithm has made important steps towards accurate and scalable federated learning.
\bibliography{reference.bib}

\appendix

\section{Convergence Analysis}
\label{sec:ConvergenceAnalysis}

In this section, we analyze convergence for \textit{STAR-stars} and then extend it to the rest of architectures.

\subsection{Convergence Analysis for \textit{STAR-stars}}

First of all, we make the following assumption for the loss function $ F^{k,i} $, as in many other relevant studies\,\cite{HierFAVG,AdaptiveFL}.

\begin{assumption}
\label{Assumption}
    For every $i$ and $k$, (1) $ F^{k,i} $ is convex;
    (2) $ F^{k,i} $ is $ \rho $-Lipschitz, i.e., $ \norm{ F^{k,i}(\w) - F^{k,i}(\w') } \le \rho\norm{ \w - \w' } $ for any $\w$ and $ \w' $; and
    (3) $ F^{k,i} $ is $ \beta $-smooth, i.e., $ \norm{ \g F^{k,i}(\w) - \g F^{k,i}(\w') } \le \beta\norm{ \w - \w' } $ for any $\w$ and $ \w' $.
\end{assumption}

Under this assumption, Lemma~\ref{lemma:Assumption} holds for the group and global loss functions. $F_k$, which is the loss function for a node group, is additionally considered here unlike \citet{AdaptiveFL}.

\begin{lemma}
\label{lemma:Assumption}
    $F$ and $F^{k}$ are convex, $ \rho $-Lipschitz, and $ \beta $-smooth.
\end{lemma}
\begin{proof}
It is straightforward from Assumption~\ref{Assumption} and the definitions of $F$ and $F^{k}$ in Definition~\ref{def:STAR-stars}.
\end{proof}

We introduce two types of intervals depending on the learning level: a \emph{g\textbf{\underline{r}}oup interval}, $ \bkt{r} \triangleq \bkt{ \p{r-1}\uptau_1, r\uptau_1 } $, indicates an interval between two successive \emph{group} aggregations, and a \emph{g\textbf{\underline{l}}obal interval}, $ \bkt{l} \triangleq \bkt{ \p{l-1}\uptau_1\uptau_2, l\uptau_1\uptau_2 } $, indicates an interval between two successive \emph{global} aggregations.

Next, we introduce the notion of \emph{group-based virtual learning} in Definition~\ref{def:VirtualLearning}, where training data is assumed to exist on a \emph{virtual} central repository for each model. This notion is used to bridge the \emph{local-to-group divergence}\,(i.e., the divergence between a local model and a group model) in a group interval and the \emph{group-to-global divergence}\,(i.e., the divergence between a group model and a global model) in a global interval.

\begin{definition}[Group-Based Virtual Learning]
\label{def:VirtualLearning}
    Given a certain group membership $\z$, for any $k$, $\bkt{r}$, and $\bkt{l}$, the \emph{virtual group model} $ \v^{k}_{\bkt{r}} $ and \emph{virtual global model} $ \v_{\bkt{l}} $ are updated by performing gradient descent steps on the centralized data examples for $\mathcal{N}^{k}$ and $\mathcal{N}$, respectively, and synchronized with the federated group model $ \w^{k} $ and the global model $\w$ at the beginning of each interval, as in Eq.~\eqref{eq:v_central}.
    \begin{equation}
    \label{eq:v_central}
    \begin{aligned}
        \v^{k}_{\bkt{r},t} \triangleq \begin{dcases}
            \w^{k}_{t} & \normal{if} ~ t = \p{r-1}\uptau_1, \\
            \v^{k}_{\bkt{r},t-1} - \eta\g F^{k}\p{\v^{k}_{\bkt{r},t-1}} & \normal{otherwise}
        \end{dcases} \\
        \v_{\bkt{l},t} \triangleq \begin{dcases}
            \w_{t} & \normal{if} ~ t = \p{l-1}\uptau_1\uptau_2, \\
            \v_{\bkt{l},t-1} - \eta\g F\p{\v_{\bkt{l},t-1}} & \normal{otherwise}
        \end{dcases}
    \end{aligned}
    \end{equation}
\end{definition}

To facilitate the interpretation, Figure~\ref{fig:wv} shows how a virtual model $\v$ is updated, following Definition~\ref{def:VirtualLearning}. For example, $\v_{\left[ l \right]} $ starts diverging from $\w$ after $(l-1)\uptau_1\uptau_2$ and becomes synchronized with $\w$ at $l\uptau_1\uptau_2$.

\begin{figure}[t!]
    \centering
    \includegraphics[width=\columnwidth]{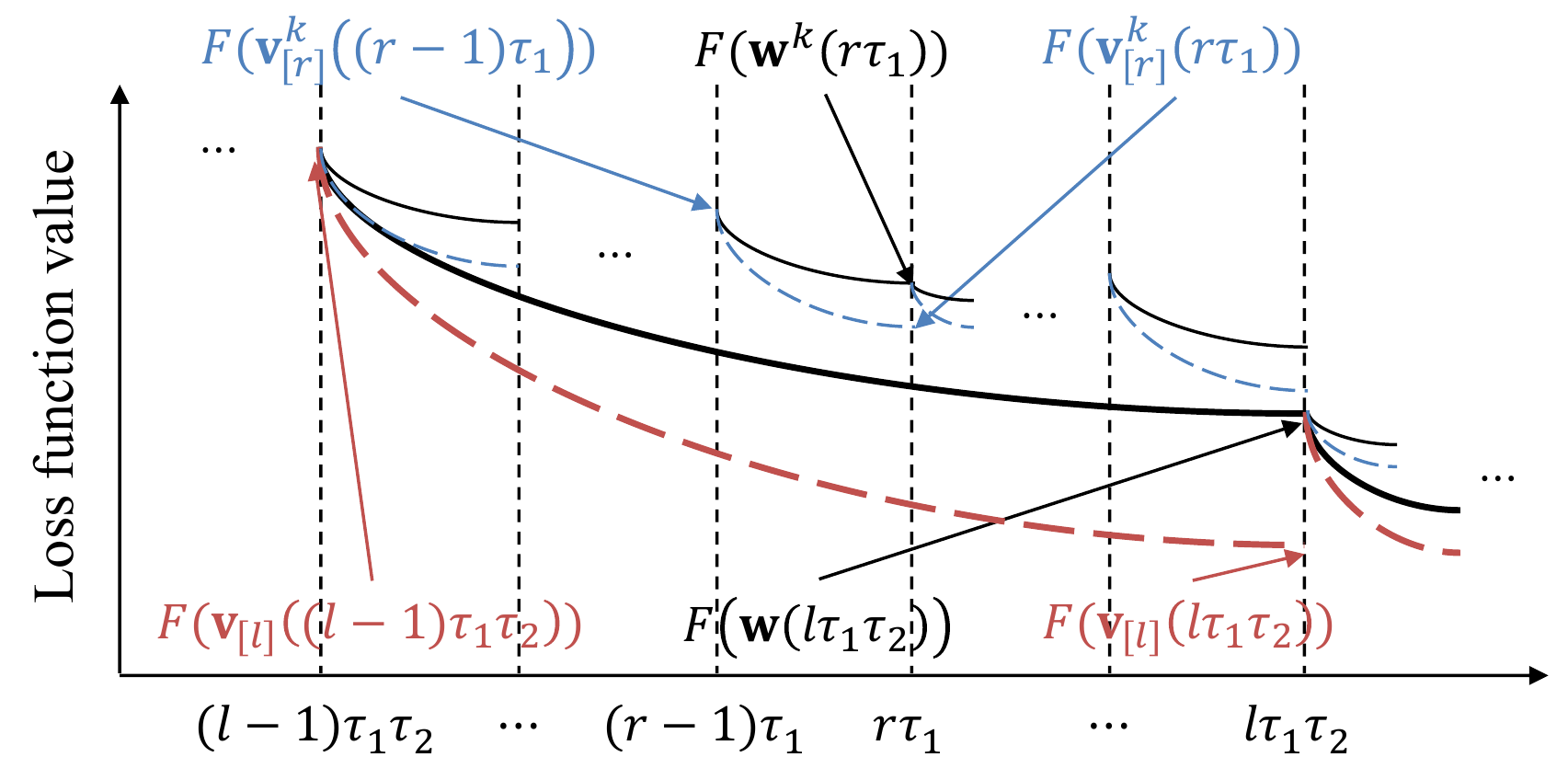}
    \caption{Illustration of loss divergence and synchronization between $ \w^{k} $ and $ \v^{k}_{\left[ r \right]} $ and between $\w$ and $ \v_{\left[ l \right]} $.}
    \label{fig:wv}
\end{figure}

Then, we formalize \emph{group-based gradient divergence} in Definition~\ref{def:GradientDivergence} that models the impact of the difference in data distributions across nodes on federated learning.
\begin{definition}[Group-Based Gradient Divergence]
\label{def:GradientDivergence}
    Given a certain group membership $\z$, for any $i$ and $k$, $ \delta^{k,i} $ is defined as the gradient difference between the $i$-th local loss and the $k$-th group loss; $ \Delta^{k} $ is defined as the gradient difference between the $k$-th group loss and the global loss, which can be expressed as Eq.~\eqref{eq:GradientDivergence}.
    \begin{equation}
    \begin{gathered}
    \label{eq:GradientDivergence}
        \delta^{k,i} \triangleq \max_{\w} \norm{ \g F^{k,i}\p{\w} - \g F^{k}\p{\w} }, \\
        \Delta^{k} \triangleq \max_{\w} \norm{ \g F^{k}\p{\w} - \g F\p{\w} }
    \end{gathered}
    \end{equation}
    Then, the \textit{local-to-group divergence} $\delta$ and the \textit{group-to-global divergence} $\Delta$ are formulated as Eq.~\eqref{eq:GradientDivergenceSum}.
    \begin{equation}
    \label{eq:GradientDivergenceSum}    
        \delta \triangleq \sum_{ k \in \gG }{ \sum_{ i \in \gN^{k} }{ \frac{ \abs{\gD^{k,i}} }{ \abs{\gD} } \delta^{k,i} } },~~  
        \Delta \triangleq \sum_{ k \in \gG }{ \frac{ \abs{\gD^{k}} }{ \abs{\gD} } \Delta^{k} }
    \end{equation}
\end{definition}

Based on Definition~\ref{def:VirtualLearning} and \ref{def:GradientDivergence}, we introduce an auxiliary lemma\,(Lemma~\ref{lemma:wi-vq}).

\begin{lemma}
\label{lemma:wi-vq}
For any $\bkt{r}$, $\bkt{l}$, and $ t \in \bkt{ \p{r-1}\uptau_1, r\uptau_1 } \subset \bkt{ \p{l-1}\uptau_1\uptau_2, l\uptau_1\uptau_2 } $, an upper bound of the norm of the difference between a local model and the virtual global model can be expressed as Eq.~\eqref{eq:node-virtual-global}.
\begin{equation}
\label{eq:node-virtual-global}
\begin{multlined}
    \norm{ \w^{k,i}_{t} - \v_{\bkt{l},t} } \le \frac{ \delta^{k,i} }{ \beta } \p{ \p{ \eta\beta + 1 }^{ t - \p{r-1}\uptau_1 } - 1 } \\
    + \frac{ \Delta^{k} }{ \beta } \p{ \p{ \eta\beta + 1 }^{ t - \p{l-1}\uptau_1\uptau_2 } - 1 }
\end{multlined}
\end{equation}

\end{lemma}
\begin{proof}
From the triangle inequality, one can simply derive Eq.~\eqref{eq:wvdistc}.
\begin{equation}
\begin{multlined}
    \norm{ \w^{k,i}_{t} - \v_{\bkt{l},t} } = \norm { \w^{k,i}_{t} - \v^{k}_{\bkt{r},t} + \v^{k}_{\bkt{r},t} - \v_{\bkt{l},t} } \\
    \le \norm{ \w^{k,i}_{t} - \v^{k}_{\bkt{r},t} } + \norm{ \v^{k}_{\bkt{r},t} - \v_{\bkt{l},t} } \label{eq:wvdistc}
\end{multlined}
\end{equation}
To conclude this proof, it thus suffices to show Eq.~\eqref{eq:wvdist1} and \eqref{eq:wvdist2}.
\begin{align}
    \norm{ \w^{k,i}_{t} - \v^{k}_{\bkt{r},t} } \le & \frac{ \delta^{k,i} }{ \beta } \p{ \p{ \eta\beta + 1 }^{ t - \p{r-1}\uptau_1 } - 1 }\quad \label{eq:wvdist1} \\
    \norm{ \v^{k}_{\bkt{r},t} - \v_{\bkt{l},t} } \le & \frac{ \Delta^{k} }{ \beta } \p{ \p{ \eta\beta + 1 }^{ t - \p{l-1}\uptau_1\uptau_2 } - 1 } \label{eq:wvdist2}
\end{align}
Then, by putting Eq.~\eqref{eq:wvdist1} and \eqref{eq:wvdist2} into Eq.~\eqref{eq:wvdistc}, we can confirm Lemma~\ref{lemma:wi-vq}.

Both Eq.~\eqref{eq:wvdist1} and \eqref{eq:wvdist2} can be easily drawn from the $\beta$-smooth property of $F^{k,i}$ and $F^{k}$. From Eq.~(4) and (6), we can derive Eq.~\eqref{eq:node-virtual-group}. 
\begin{equation}
\label{eq:node-virtual-group}
\begin{aligned}
    & \norm{ \w^{k,i}_{t} - \v^{k}_{\bkt{r},t} } \\
    & = \norm{ \w^{k,i}_{t-1} - \eta\g F^{k,i}\p{\w^{k,i}_{t-1}} - \v^{k}_{\bkt{r},t-1} + \eta\g F^{k}\p{\v^{k}_{\bkt{r},t-1}}} \\
    & \begin{multlined} \le \norm{ \w^{k,i}_{t-1} - \v^{k}_{\bkt{r},t-1} } \\
    + \eta\norm{ \g F^{k,i}\p{\w^{k,i}_{t-1}} - \g F^{k,i}\p{\v^{k}_{\bkt{r},t-1}} } \\
    + \eta\norm{ \g F^{k,i}\p{\v^{k}_{\bkt{r},t-1}} - \g F^{k}\p{\v^{k}_{\bkt{r},t-1}} } \end{multlined} \\
    & \le \p{ \eta\beta + 1 }\norm{ \w^{k,i}_{t-1} - \v^{k}_{\bkt{r},t-1} } + \eta\delta^{k,i}
\end{aligned}
\end{equation}

The last inequality stems from the $ \beta $-smoothness of $ F^{k,i} $ and Definition~\ref{def:GradientDivergence}.

Then, since $\w^{k,i}_{t} = \w^{k}_{t} = \v^{k}_{\bkt{r},t} $ at every group aggregation from Eq.~(4) and (6), Eq.~\eqref{eq:node-virtual-group} can be rewritten as Eq.~\eqref{eq:wi-delta}. 
\begin{align}
    \norm{ \w^{k,i}_{t} - \v^{k}_{\bkt{r},t} } \le & \eta\delta^{k,i}\sum^{ t - \p{r-1}\uptau_1 }_{ y = 1 }{ \p{ \eta\beta + 1 }^{ y - 1 } } \nonumber \\
    =& \frac{ \delta^{k,i} }{ \beta } \p{ \p{ \eta\beta + 1 }^{ t - \p{r-1}\uptau_1 } - 1 } \label{eq:wi-delta}
\end{align}
Analogously, one can derive Eq.~\eqref{eq:wvdist2}. This is the end of the proof of Lemma~\ref{lemma:wi-vq}.
\end{proof}

From Lemma~\ref{lemma:wi-vq} and Jensen's inequality, for all $t$ and $q$, we have Eq.~\eqref{eq:w-v}.
\begin{equation}
\label{eq:w-v}
\begin{multlined}
    \norm{ \w_{t} - \v_{\bkt{l},t} } \le \sum_{ k \in \gG } \sum_{ i \in \gN^{k} } \frac{ \abs{\gD^{k,i}} }{ \abs{\gD} } \norm{ \w^{k,i}_{t} - \v_{\bkt{l},t} } \\
    \le \frac{ \delta }{ \beta } \p{ \p{ \eta\beta + 1 }^{ \uptau_1 } - 1 } + \frac{ \Delta }{ \beta } \p{ \p{ \eta\beta + 1 }^{ \uptau_1\uptau_2 } - 1 }
\end{multlined}
\end{equation}
Finally, for \textit{STAR-stars}, we derive the convergence bound between the federated global model and the virtual global model by Theorem~\ref{def:Conv-STAR-stars}.

\begin{theorem}[Convergence Bound of \textit{STAR-stars}]
\label{def:Conv-STAR-stars}
For any global interval $\bkt{l}$ and $ t \in \bkt{l} $, if $ F^{k,i} $ is $ \beta $-smooth for every $i$ and $k$ in Eq.~\eqref{eq:GradientDivergence}, then Eq.~\eqref{eq:Conv-STAR-stars} holds.
\begin{equation}
\label{eq:Conv-STAR-stars}
\begin{split}
    F\p{\w_{t}} - F\p{\v_{\bkt{l},t}} \le \frac{ \rho }{ \beta }\p{ \delta h\p{\uptau_1} + \Delta h\p{\uptau_1\uptau_2} } \\
    \normal{where} ~ h\p{t} \triangleq \p{ \eta\beta + 1 }^{t} - 1
\end{split}
\end{equation}
\end{theorem}

\subsection{Convergence Analysis for the Other Architectures}

In this section we analyze convergence bounds for flat architectures\,(\textit{STAR} and \textit{RING}), consensus group architectures\,(\textit{STAR-rings}, \textit{RING-stars}, and \textit{RING-rings}, and pluralistic group architectures\,(\textit{stars} and \textit{rings}). First, the convergence bound of \textit{STAR} is the same as that of \textit{STAR-stars} with no group $\abs{\gG}=1$ and no group communication $\uptau_2=1$, as in Eq.~\eqref{eq:Conv-STAR}. By excluding the notion of groups, it also means that the group-to-global divergence $\Delta$ becomes $0$ and thus the local-to-group divergence $\delta$ becomes the local-to-global divergence $D$, which is extended from Eq.~\eqref{eq:GradientDivergence} and \eqref{eq:GradientDivergenceSum}.

\begin{equation}
\label{eq:Conv-STAR}
\begin{aligned}
F\p{\w_{t}} - F\p{\v_{\bkt{l},t}} & \le \frac{ \rho }{ \beta } D h\p{\uptau} \\
\normal{where} & ~ D \triangleq \sum_{ i \in \gN }{ \frac{ \abs{\gD^{i}} }{ \abs{\gD} } D^{i} } \\
\normal{and} & ~ D^{i} \triangleq \max_{\w} \norm{ \g F^{i}\p{\w} - \g F\p{\w} }
\end{aligned}
\end{equation}

Next, the convergence bound of \textit{stars} is the same as multiple independent \textit{STAR} groups, where the total node size of \textit{stars} is the sum of nodes of each \textit{STAR} group. Thus, by regarding the local-to-global divergence $D$ from Eq.~\eqref{eq:Conv-STAR} of each \textit{STAR} group as the local-to-group divergence $\delta^{k}$ of each group in \textit{stars}, we have Eq.~\eqref{eq:Conv-stars} for \textit{stars}.

\begin{equation}
\label{eq:Conv-stars}
F\p{\w_{t}} - F\p{\v_{\bkt{l},t}} \le \frac{ \rho }{ \beta } \delta h\p{\uptau}
\end{equation}

Next, for \textit{STAR-rings}, similar to Theorem~\ref{thm:RING}, each \textit{ring}-based learning in a group is an unbiased estimator of the centralized learning within the group because of Eq.~\eqref{eq:STAR-rings-CL}.
\begin{equation}
\label{eq:STAR-rings-CL}
    \gE\bkt{\g F^{k,i}\p{\w_{t}}} = \sum_{i \in \gN^{k}}{\frac{\abs{\gD^{k,i}}}{\abs{\gD^{k}}}\g F^{k,i}\p{\w_{t}}} = \g F^{k}\p{\w_{t}}
\end{equation}
Then, by approximating $\g F^{k,i}\p{\w_{t}}$ of Eq.~\eqref{eq:GradientDivergence} to $\gE\bkt{\g F^{k,i}\p{\w_{t}}}$ of Eq.~\eqref{eq:STAR-rings-CL}, the local-to-group divergence $\delta$ becomes $0$ and thus the group-to-global divergence $\Delta$ becomes the local-to-global divergence $D$. Thus, Eq.~\eqref{eq:Conv-STAR-stars} is extended to Eq.~\eqref{eq:Conv-STAR-rings} for \textit{STAR-rings}.
\begin{equation}
\label{eq:Conv-STAR-rings}
F\p{\w_{t}} - F\p{\v_{\bkt{l},t}} \le \frac{ \rho }{ \beta } D h\p{\uptau_1\uptau_2}
\end{equation}

Next, for \textit{RING-stars}, similar to Theorem~\ref{thm:RING}, the global \textit{RING}-based learning is an unbiased estimator of the globally centralized learning in consideration of each \textit{stars}-based learning in a group because of Eq.~\eqref{eq:RING-stars-CL}.
\begin{equation}
\label{eq:RING-stars-CL}
    \gE\bkt{\g F^{k}\p{\w_{t}}} = \sum_{k \in \gG}{\frac{\abs{\gD^{k}}}{\abs{\gD}}\g F^{k}\p{\w_{t}}} = \g F\p{\w_{t}}
\end{equation}
Then, by approximating $\g F^{k}\p{\w_{t}}$ of Eq.~\eqref{eq:GradientDivergence} to $\gE\bkt{\g F^{k}\p{\w_{t}}}$ of Eq.~\eqref{eq:RING-stars-CL}, the group-to-global divergence $\Delta$ becomes $0$ and thus the local-to-group divergence $\delta$ becomes the local-to-global divergence $D$. Thus, Eq.~\eqref{eq:Conv-STAR-stars} is extended to Eq.~\eqref{eq:Conv-RING-stars} for \textit{RING-stars}.
\begin{equation}
\label{eq:Conv-RING-stars}
F\p{\w_{t}} - F\p{\v_{\bkt{l},t}} \le \frac{ \rho }{ \beta } D h\p{\uptau_1}
\end{equation}

Analogously, from Theorem~\ref{thm:RING}, \textit{RING}, \textit{Ring-rings}, and \textit{rings} can be easily shown to have the approximate convergence bound of $0$.

\section{Communication Scalability}
\label{sec:CommScalability}

In this section, we provide an approach to measure communication scalability. For the analysis, $M$ denotes the model size; and given the total learning steps $T$, $\uptau_{f} \triangleq T/\uptau$ is defined as the number of global communications for the flat architectures; $\uptau_{c} \triangleq T/\uptau_1\uptau_2$ is defined as the number of global communications for the consensus group architectures; $\uptau_{p} \triangleq T/\uptau$ is defined as the number of group communications for the pluralistic group architectures.

\textit{STAR} sends $M \abs{\gN} \uptau_{f}$ of total global aggregation data; \textit{RING} sends $M \uptau_{f}$ of total global inter-node transfer because only one node is active for the diurnal property; \textit{STAR-stars} sends $M \abs{\gN} \p{\uptau_2-1}\uptau_{c}$ of total group aggregation data and $M \abs{\gN} \uptau_{c}$ of total global aggregation data, thus $M \abs{\gN} \uptau_2\uptau_{c}$ of total communication data, which is the same cost as \textit{STAR} in case of $\uptau_{f}=\uptau_2\uptau_{c}$ as suggested by \citet{HierFAVG}; \textit{stars} sends $M \abs{\gN} \uptau_{p}$ of total group aggregation data, which is the same cost as \textit{STAR} because $\uptau_{p}=\uptau_{f}$ from the definition; analogously, one can show the total communication data size for the rest of architectures, as summarized in Table~\ref{tbl:Architecture}.

\section{Deferred Proofs}
\label{sec:Proof}

\setcounter{theorem}{0}

\begin{theorem}
\textit{RING} is an unbiased estimator of the centralized learning that learns a centralized model by assuming the federated datasets to be located at a centralized storage.
\end{theorem}
\begin{proof}
The centralized learning is defined as Eq.~\eqref{eq:w_CL}.
\begin{equation}
\label{eq:w_CL}
    \w_{t} = \w_{t-1} - \eta\g F\p{\w_{t-1}}
\end{equation}
Next, for \textit{RING} update, we regard the model and data communication relationship as the opposite, that is, instead of \textit{RING} transferring a local model from one node to another while data stays in place, \textit{RING} is redefined as switching data from one node to another while the local model stays in a certain node. Thus, at the time $t$ of data transfer from node $i$ to the certain node, Eq.~\eqref{eq:w_RING} changes to Eq.~\eqref{eq:redefined_w_RING}. Note that the index for the certain node is not denoted because it does not need be distinguished from the others.
\begin{equation}
\label{eq:redefined_w_RING}
    \w_{t} = \w_{t-1} - \eta\g F^{i}\p{\w_{t-1}}
\end{equation}
Thus, Eq.~\eqref{eq:redefined_w_RING} equals the centralized learning in expectation because of Eq.~\eqref{eq:RING-CL}
\begin{equation}
\label{eq:RING-CL}
    \gE\bkt{\g F^{i}\p{\w_{t-1}}} = \sum_{i \in \gN}{\frac{\abs{\gD^{i}}}{\abs{\gD}}\g F^{i}\p{\w_{t-1}}} = \g F\p{\w_{t-1}}
\end{equation}
\end{proof}

\begin{theorem}
The convergence bound O($\delta h(\uptau)$) of \textit{stars} doesn't necessarily be better than O($D h(\uptau)$) of \textit{STAR}.
\end{theorem}
\begin{proof}
From Eq.~\eqref{eq:GradientDivergence} and triangle inequality, we have Eq.~\eqref{eq:delta-bound}.
\begin{equation}
\label{eq:delta-bound}
\begin{aligned}
    & \delta^{k,i} = \max_{\w} \norm{ \g F^{k,i}\p{\w} - \g F^{k}\p{\w} } \\
    & = \max_{\w} \norm{ \g F^{k,i}\p{\w} - \g F\p{\w} + \g F\p{\w} - \g F^{k}\p{\w} } \\
    & \le \max_{\w} \bkt{ \norm{ \g F^{k,i}\p{\w} - \g F\p{\w} } + \norm{ \g F\p{\w} - \g F^{k}\p{\w} } }
\end{aligned}
\end{equation}
By summing Eq.~\eqref{eq:delta-bound} for all $i$ and $k$ and considering Eq.~\eqref{eq:GradientDivergenceSum} and \eqref{eq:Conv-STAR}, we have Eq.~\eqref{eq:delta-sum-bound}.
\begin{equation}
\label{eq:delta-sum-bound}
    \delta \le \Delta + D
\end{equation}

Similarly, from Eq.~\eqref{eq:Conv-STAR} and triangle inequality, we have Eq.~\eqref{eq:D-bound}.
\begin{equation}
\label{eq:D-bound}
\begin{aligned}
    & D^{i} = \max_{\w} \norm{ \g F^{i}\p{\w} - \g F\p{\w} } \\
    & = \max_{\w} \norm{ \g F^{i}\p{\w} - \g F^{k}\p{\w} + \g F^{k}\p{\w} - \g F\p{\w} } \\
    & \le \max_{\w} \bkt{ \norm{ \g F^{i}\p{\w} - \g F^{k}\p{\w} } + \norm{ \g F^{k}\p{\w} - \g F\p{\w} } }
\end{aligned}
\end{equation}
By summing Eq.~\eqref{eq:D-bound} for all $i$ and considering Eq.~\eqref{eq:GradientDivergenceSum}, we have Eq.~\eqref{eq:D-sum-bound}.
\begin{equation}
\label{eq:D-sum-bound}
    D \le \delta + \Delta
\end{equation}
Lastly, based on Eq.~\eqref{eq:delta-sum-bound} and \eqref{eq:D-sum-bound}, we can infer that the worst case of $\delta$ equals $\Delta + D$ that is larger than $D$, in which \textit{STAR} achieves lower convergence bound than \textit{stars}.
\end{proof}

\begin{theorem}
\textit{RING} exhibits higher variance than the centralized learning\,(unbiased estimator of \textit{RING} from Theorem~\ref{thm:RING}).
\end{theorem}
\begin{proof}
First, from the $\beta$-smoothness of $F$ and Eq.~\eqref{eq:w_CL}, we have Eq.~\eqref{eq:bound_w_CL} for centralized learning.

\begin{equation}
\begin{aligned}
\label{eq:bound_w_CL}
F\p{\w_{t}} & = F\p{\w_{t-1} - \eta\g F\p{\w_{t-1}}} \\
& \le F\p{\w_{t-1}} - \eta\p{1-\frac{\eta\beta}{2}}\norm{\g F\p{\w_{t-1}}}^2
\end{aligned}
\end{equation}

Next, from the $\beta$-smoothness of $F$ and the definition of \textit{RING} update in Eq.~\eqref{eq:redefined_w_RING}, we have Eq.~\eqref{eq:bound_w_RING} for \textit{RING}.

\begin{equation}
\begin{aligned}
\label{eq:bound_w_RING}
    F\p{\w_{t}} & = F\p{\w_{t-1} - \eta\g F^{i}\p{\w_{t-1}}} \\
    & \begin{multlined} \le F\p{\w_{t-1}} + \g F\p{\w_{t-1}}\p{-\eta\g F^{i}\p{\w_{t-1}}} \\
    + \frac{\eta^2\beta}{2}\norm{\g F^{i}\p{\w_{t-1}}} \end{multlined}
\end{aligned}
\end{equation}

In expectation with regard to $i$, we have Eq.~\eqref{eq:bound_w_RING_2}.
\begin{equation}
\begin{aligned}
\label{eq:bound_w_RING_2}
F\p{\w_{t}} & \begin{multlined}[t] \le F\p{\w_{t-1}} - \eta\norm{\g F\p{\w_{t-1}}}^2 \\
+ \frac{\eta^2\beta}{2}\gE_{i}{\norm{\g F^{i}\p{\w_{t-1}}}} \end{multlined} \\
& \begin{multlined} \le F\p{\w_{t-1}} - \eta\p{1-\frac{\eta\beta}{2}}\norm{\g F\p{\w_{t-1}}}^2 \\
+ \frac{\eta^2\beta}{2}\bkt{\gE_{i}\norm{\g F^{i}\p{\w_{t-1}}} - \norm{\gE_{i}\g F^{i}\p{\w_{t-1}}}^2} \end{multlined}
\end{aligned}
\end{equation}
%
where $\gE_{i}\norm{\g F^{i}\p{\w_{t-1}}} - \norm{\gE_{i}\g F^{i}\p{\w_{t-1}}}^2$ is the learning variance of \textit{RING}, which is an added term from Eq.~\eqref{eq:bound_w_CL}.

Analogously, one can show similar variances for \textit{STAR-rings}, \textit{RING-stars}, \textit{Ring-rings}, and \textit{rings}.
\end{proof}

\section{TornadoAggregate Details}
\label{sec:AlgorithmDetail}

\begin{algorithm}[t!]
\SetAlgoLined
\SetKwInOut{Input}{\textsc{Input}}
\SetKwInOut{Output}{\textsc{Output}}
\caption{Tornadoes\,(STAR-rings)}
\label{alg:Tornadoes}
\Input{$\gN$, $\abs{\gG}$, $C$, $\uptau_1$, $\uptau_2$}
\Output{$\w_{T}$}
Initialize $\set{\w^{k,i}_{0}}_{i \in \gN}$ to a random model $\w_{0}$ \\
Initialize a random ring $\bkt{i_{j} \in \gN | j \in \mathbb{N}^0, i_{j+\abs{\gN}} = i_{j}}$ \\
$\set{\gN^{k}}_{k \in \gG} \leftarrow \textsc{Cluster}\p{\gN}$ \tcp{Algorithm \ref{alg:Grouping}}
\For{$t \leftarrow 0, \cdots, T-1$}{
    \For{\normal{\textbf{each}} $k \in \gG$ \normal{\textbf{in parallel}}}{
        \For{\normal{\textbf{each}} $c \leftarrow 0, \cdots, C-1$ \normal{\textbf{in parallel}}}{
            $j \leftarrow \floor{t / \uptau_1}$, $i \leftarrow \p{i_{j}+c} \normal{ mod } \abs{\gN^{k}}$ \\
            $\w^{k,i}_{t+1} \leftarrow \w^{k,i}_{t} - \eta\g F^{k,i}\p{\w^{k,i}_{t}}$ \\
            \If{$t \normal{ mod } \uptau_1 \normal{ \textbf{and} } t \normal{ mod } \uptau_1\uptau_2 \ne 0$}{
                $i_{next} \leftarrow \p{i_{j+1}+c} \normal{ mod } \abs{\gN^{k}}$ \\
                $\w^{k,i_{next}}_{t} \leftarrow \w^{k,i}_{t}$
            }
        }
    }
    \If{$t \normal{ mod } \uptau_1\uptau_2 = 0$}{
        $\set{\w^{k,i}_{t}}_{k \in \gG, i \in \gN} \leftarrow \sum_{k \in \gG}{\sum_{i \in \gN^{k}}{\frac{\abs{\gD^{k,i}}}{\abs{\gD}}\w^{k,i}_{t}}}$
    }
}
\end{algorithm}

Algorithm \ref{alg:Tornadoes} shows the overall procedure of Tornadoes that takes the node set $\gN$ and the number of groups $\abs{\gG}$, chains $C$, epochs $\uptau_1$, and communication rounds $\uptau_2$ as input and returns the final model $\w_{T}$ as output. It begins by initializing all local models $\w^{k,i}_0$, a randomly permuted inter-node ring $\bkt{i_{j}}$, and group indices $\set{\gN^{k}}$ by clustering nodes\,(Lines 1--3). Then, for each group and each chain of the group, the local updates are performed at the node $i$\,(Lines 5--8); every $\uptau_1$ epochs, each local model is transferred to the next node $i_{next}$ within the same group $k$\,(Lines 9--11); every $\uptau_1\uptau_2$ steps, the global model is learned by aggregating all local models and then broadcasts back to all nodes\,(Lines 12--13). Overall, Lines 4--13 repeat for $T$ steps.

We derive another heuristic of TornadoAggregate with \textit{rings} architecture, called Tornado-rings, which is the same as Tornadoes without the global aggregation to develop a independent and specialized model for each group, i.e., Lines 12--13 of Algorithm \ref{alg:Tornadoes} are not executed for Tornado-rings. We note that, for the \textit{stars} and \text{rings}, the test performance are measured with the group model of each independent group.

\begin{algorithm}[t!]
\SetAlgoLined
\SetKwProg{Function}{function}{:}{}
\caption{Grouping Scheme}
\label{alg:Grouping}
\Function{\textsc{Group\_By\_IID}($\gN$)}{
    $\textsc{Cost}_{A}\p{i, k} \triangleq \textsc{EMD}\p{\gD^{k},\gD}$ \\
    $\textsc{Cost}_{U}\p{i, k} \triangleq \textsc{EMD}\p{\gD^{k,i}, \gD}$ \\
    \KwRet \textsc{Group}($\gN$, $\textsc{Cost}_{A}$, $\textsc{Cost}_{U}$)
}
\texttt{\\}
\Function{\textsc{Cluster}($\gN$)}{
    $\textsc{Cost}_{A}\p{i, k} \triangleq \textsc{EMD}\p{\gD^{k,i},\gD^{k}}$ \\
    $\textsc{Cost}_{U}\p{i, k} \triangleq \textsc{EMD}\p{\gD^{k,i}, \gD^{k}}$ \\
    \KwRet \textsc{Group}($\gN$, $\textsc{Cost}_{A}$, $\textsc{Cost}_{U}$)
}
\texttt{\\}
\Function{\textsc{Group}($\gN$, $\textsc{Cost}_{A}$, $\textsc{Cost}_{U}$)}{
    Select random medoid nodes $\gN_{m}$ of size $\abs{\gG}$ \\
    $\z \leftarrow \bkt{ \argmin_{ k \in \gG }{ \textsc{Cost}_{A}\p{i, k } } | \forall i \in \gN }$ \\
    \While{\normal{the last} $\textsc{Cost}_{A}$ \normal{is not steady}}{
        $ \gN_{m} \leftarrow \bkt{ \argmin_{ i \in \gN^{k} }{ \textsc{Cost}_{U}\p{i, k} } | \forall k \in \gG } $ \\
        $ \z \leftarrow \bkt{ \argmin_{ k \in \gG }{ \textsc{Cost}_{A}\p{i, k } } | \forall i \in \gN } $
    }
    \KwRet $\set{\set{i | (i, k) \in \z, k={k}^{\prime}}|{k}^{\prime} \in \gG}$
}
\end{algorithm}

Algorithm \ref{alg:Grouping} shows the \emph{two} grouping schemes: \textsc{Group\_By\_IID} and \textsc{Cluster}. Both functions define their own association cost $\textsc{Cost}_{A}$ and update cost $\textsc{Cost}_{U}$ and, in turn, call \textsc{Group} function with the defined costs. The costs are based on the EMD\,(earth mover distance) that can approximately model the learning divergences, as proposed by \citet{IIDSharing}, which can be expressed as Eq.~\eqref{eq:EMD}. In \textsc{Group\_By\_IID} function\,(Lines 1--4), a group data distribution $\gD^{k,i}$ is compared with the global dataset $\gD$ to improve the group-to-global divergence $\Delta$ of Eq.~\eqref{eq:GradientDivergence}, while in \textsc{Cluster} function\,(Lines 6--9), a local data distribution $\gD^{k,i}$ is compared with a group data distribution $\gD^{k}$ to improve the local-to-group divergence $\delta$ of Eq.~\eqref{eq:GradientDivergence}. It should be noted that for the $\textsc{Cost}_{U}$ of \textsc{Group\_By\_IID} function, we had no choice but to use $\gD^{k,i}$ instead of $\gD^{k}$ because a cost related to a node should be returned to determine a new medoid node.
\begin{equation}
\label{eq:EMD}
\begin{aligned}
& \textsc{EMD}\p{\gD_{1}, \gD_{2}} \\
& \triangleq \sum_{\forall class}{ \abs{ \gP\p{ y_{j} = class | j \in \gD_1 } - \gP\p{ y_{j} = class | j \in \gD_2 } } }
\end{aligned}
\end{equation}

The \textsc{Group} function aims at finding subsets of node indexes for all groups $\set{\gN^{k}}_{k=1 \ldots \abs{\gG}}$ such that it reduces the defined costs to the extent possible. For this purpose, it begins by selecting random medoid nodes $\gN_{m}$ of size $\abs{\gG}$\,(Line 12). Then, it iteratively updates $\z$ by minimizing $\textsc{Cost}_{A}$ for all nodes and $\textsc{Cost}_{U}$ for all groups until the cost is steady\,(Lines 14--16).

\begin{table}[t!]
\centering
\small
    \begin{tabular}{ccc} \toprule
        \textbf{Dataset} & \textbf{Initial Cost} & \textbf{Final Cost} \\ \midrule
        FedShakespeare & 0.391 & \makecell{0.375\\(4.3\% reduced)} \\
        MNIST & 0.728 & \makecell{0.474\\(53.6\% reduced)} \\ \bottomrule
    \end{tabular}
\caption{\textbf{Reduction of clustering cost in Algorithm~\ref{alg:Grouping}}.}
\label{tbl:ClusteringCost}
\end{table}

\begin{table*}[t!]
\centering
\small
\begin{tabular}{ccccccc} \toprule
    \textbf{Hierarchy} & \textbf{Algorithm} & \textbf{Architecture} & \makecell{\textbf{Grouping}\\\textbf{Scheme}} & \makecell{\textbf{Group}\\\textbf{Size}} & \textbf{\# Chain} & \makecell{\textbf{Communication}\\\textbf{Interval}} \\ \midrule
    Flat & FedAvg\,\cite{FedAvg} & \textit{STAR} & - & 1 & - & $\uptau=100$ \\ \midrule
    \multirow{5}{*}{\makecell{Consensus\\Group}} & HierFAVG\,\cite{HierFAVG} & \textit{STAR-stars} & Random & 5 & - & $\uptau_1=10,\uptau_2=10$ \\
    & Astraea\,\cite{Astraea} & \textit{STAR-rings} & IID & 2 & 1 & $\uptau_1=10,\uptau_2=10$ \\
    & MM-PSGD\,\cite{MM-PSGD} & \textit{RING-stars} & Cluster & 10 & 1 & $\uptau_1=10,\uptau_2=10$ \\
    & \textbf{Tornado\,(Proposed)} & \textit{RING-stars} & IID & 2 & 2 & $\uptau_1=10,\uptau_2=10$ \\
    & \textbf{Tornadoes\,(Proposed)} & \textit{STAR-rings} & Cluster & 10 & 10 & $\uptau_1=10,\uptau_2=10$ \\ \midrule
    \multirow{2}{*}{\makecell{Pluralistic\\Group}} & IFCA\,\cite{IFCA}) & \textit{stars} & Cluster & 10 & - & $\uptau=100$ \\
    & SemiCyclic\,\cite{SemiCyclic} & \textit{rings} & Random & 5 & 1 & $\uptau=100$ \\
    & \textbf{Tornado-rings\,(Proposed)} & \textit{rings} & Cluster & 10 & 10 & $\uptau=100$ \\ \bottomrule
\end{tabular}
\caption{\textbf{Algorithm parameters.}}
\label{tbl:Parameters}
\end{table*}

\section{Supplementary Evaluation}
\label{sec:SupplementaryEvaluation}

\subsection{Experimental Setting}

\subsubsection{Configuration} We used FedML\,\cite{chaoyanghe2020fedml}, one of the most widely used simulation frameworks for federated learning, on PyTorch 1.6.0 to extensively evaluate the performance of various datasets, models, and algorithms.

\subsubsection{Parameters} The parameters for both FedShakespeare on RNN and MNIST on logistic regression benchmarks followed those suggested by FedML. The benchmarks used SGD\,(Stochastic Gradient Descent) optimizer with the learning rate of 0.03. In addition, we randomly sampled 100 nodes for both train and test phase, out of 715 nodes for FedShakespeare and 1000 nodes for MNIST.

Table~\ref{tbl:Parameters} shows the parameters used for each algorithm. In particular, for the group size, we applied the aforementioned \emph{small ring} principle to all algorithms such that the group size of an algorithm with IID node grouping, random grouping, and node clustering is set to 2, 5, and 10, respectively, where 10 is considered a reasonably large value for the group size; for the number of chains, we applied the \emph{ring chaining} principle to the proposed TornadoAggregate heuristics such that the number of chains is set to the number of groups, which is the maximum value by definition; for the communication interval, we firstly determined the product of $\uptau_1$ and $\uptau_2$ of HierFAVG to be equal to $\uptau$ of FedAVG so that HierFAVG can improve accuracy by sacrificing little communication cost, as suggested by \citet{HierFAVG}, and then we set the same parameters as HierFAVG for the rest of algorithms.

\subsection{Additional Results}

\begin{figure*}[h!]
\centering
    \includegraphics[width=.5\linewidth]{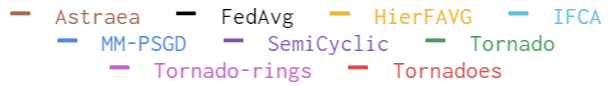} \\
    \begin{minipage}[t]{.4\linewidth}
        \includegraphics[width=\linewidth]{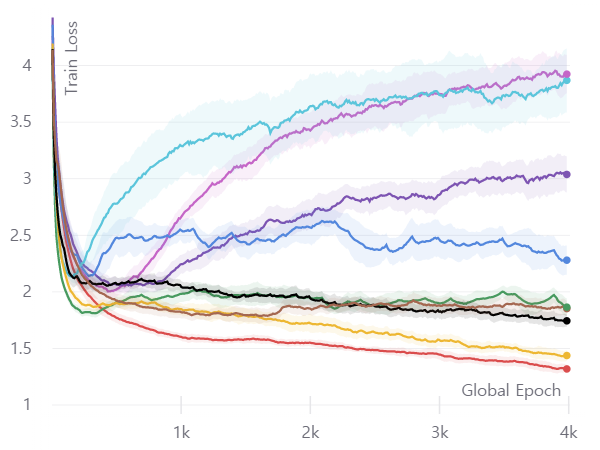}
        \subcaption{Train Loss.}
    \end{minipage}
    \begin{minipage}[t]{.4\linewidth}
        \includegraphics[width=\linewidth]{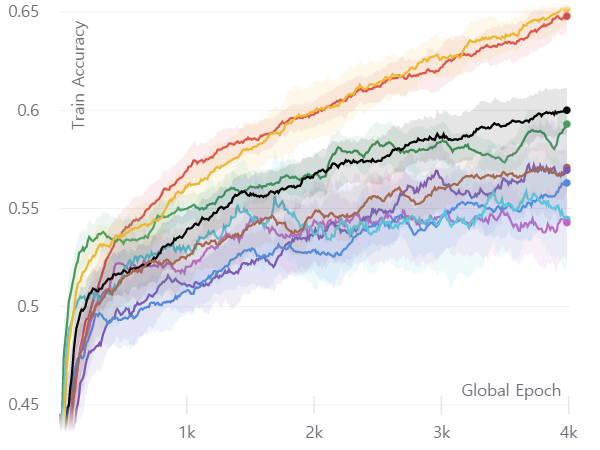}
        \subcaption{Train Accuracy.}
    \end{minipage}
\caption{\textbf{FedShakespeare}.}
\label{fig:FedShakespeare-Train}
\end{figure*}

Figure~\ref{fig:FedShakespeare-Train} shows the train loss and accuracy of \emph{nine} algorithms on FedShakespeare dataset. Even though HierFAVG seemingly outrun the others, compared with the test accuracy of HierFAVG in Figure~\ref{fig:FedShakespeare-TestAccuracy}, we can infer that it overfit towards the training dataset. Similar to the aforementioned results of IFCA, Tornado-rings performed bad because of the small reduction of clustering cost, defined in Algorithm~\ref{alg:Grouping}, for FedShakespeare dataset, as shown in Table~\ref{tbl:ClusteringCost}. Low accuracy of SemiCyclic algorithm can be attributed to the low data utilization with low number of active nodes, which is also pointed out by \citet{MM-PSGD}.

\begin{figure*}[t!]
    \centering
    \includegraphics[width=.5\linewidth]{figure/legend.png} \\
    \begin{minipage}[t]{.4\linewidth}
        \includegraphics[width=\linewidth]{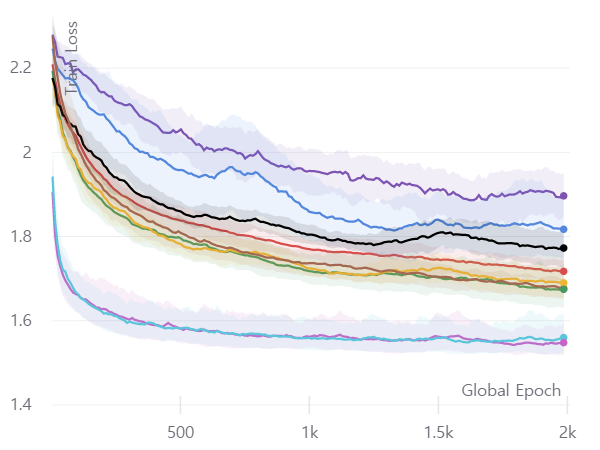}
        \subcaption{Train Loss.}
    \end{minipage}
    \begin{minipage}[t]{.4\linewidth}
        \includegraphics[width=\linewidth]{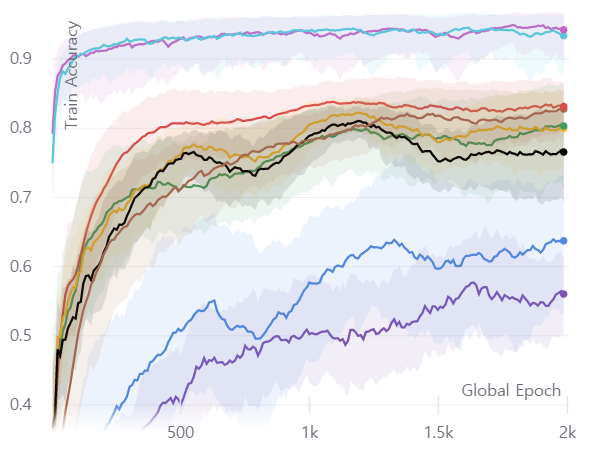}
        \subcaption{Train Accuracy.}
    \end{minipage}
    \caption{\textbf{MNIST}.}
\label{fig:MNIST-Train}
\end{figure*}

Figure~\ref{fig:MNIST-Train} shows the train loss and accuracy of all algorithms on MNIST dataset. Interestingly, in contrast to the results for FedShakespeare dataset, Tornado-rings significantly outperformed the others except for closely following IFCA. The reason why the worst performers became the best performers can also be attributed to the large reduction clustering cost, as shown in Table~\ref{tbl:ClusteringCost}. To strike the balance between the two extremes, we leave Tornado-rings as our future work. Aside from Tornado-rings and IFCA, Tornadoes outperformed the others and the rest of algorithms exhibited the similar performance trend to that from FedShakespeare.




\section{Future Directions}

We consider the following works orthogonal to our work, which can thus be easily extended to by TornadoAggregate.

\begin{itemize}
\item
    \textbf{Communication Reduction Techniques}: This category includes quantization\,\cite{FSVRG,konevcny2016federatedLearning}, compression\,\cite{sattler2019robust}, or dropout\,\cite{caldas2018expanding}.
\item
    \textbf{Communication-Aware Learning}: This category includes adaptive communication interval\,\cite{AdaptiveFL}, communication-constrained learning\,\cite{FedCS,HybridFL}, or multi-objective optimization of learning error and communication\,\cite{zhu2019multi}.
\item
    \textbf{Global-Information Sharing}: This category includes sharing a subset of global IID data samples\,\cite{IIDSharing,HybridFL}, sharing a subset of global data features to scale up the feature-related parameters of a local optimizer\,\cite{FSVRG}, or sharing a generative model that can produce an augmented IID dataset\,\cite{FedMultiTask, jeong2018communication}.
\end{itemize}

On the other hand, we aim at improving TornadoAggregate in the following directions.
\begin{itemize}
\item
    \textbf{Client Sampling}: Client sampling techniques\,\cite{FedAvg,FedProx,SampledFL} introduce a different variance aspect from those handled in this study.
\item
    \textbf{Peer-to-Peer Learning}: Other than the star and ring architectures, peer-to-peer\,(P2P) federated learning\,\cite{MATCHA,hegedHus2019gossip,MATCHA} should also be considered.
\item
    \textbf{Transfer Learning}: In all synchronizations, model parameters are set to the previously learned model parameters, but we can also consider transferring parameters between different types of models\,\cite{jeong2018communication,li2019fedmd,he2020group}.
\item
    \textbf{Continual Learning}: We can extend TornadoAggregate to the novel techniques in the field of continual learning, such as weight decomposition\,\cite{yoon2020federated} or loss regularization\,\cite{shoham2019overcoming}.
\item
    \textbf{Algorithm Optimization}: Other than EMD of Eq.~\eqref{eq:EMD}, IIDness can also be quantified by loss divergence\,\cite{SampledFL}, gradient divergence\,\cite{AdaptiveFL}, or weight divergence\,\cite{IIDSharing}.
\end{itemize}

\end{document}